\def\year{2021}\relax
\documentclass[letterpaper]{article} 
\usepackage{aaai21}  
\usepackage{times}  
\usepackage{helvet} 
\usepackage{courier}  
\usepackage[hyphens]{url}  
\usepackage{natbib}  
\usepackage{caption} 
\usepackage{amsmath,amsfonts,amssymb,amsthm, mathtools}
\usepackage{booktabs}
\usepackage[linesnumbered, ruled, vlined, norelsize]{algorithm2e}
\usepackage{diagbox}
\usepackage{multirow}

\newtheorem{theorem}{Theorem}

\newtheorem{corollary}[theorem]{Corollary}
\theoremstyle{definition}
\newtheorem{definition}[theorem]{Definition}
\newtheorem{observation}[theorem]{Observation}

\DeclareMathOperator*{\argmin}{arg\,min}
\usepackage{color-edits}
\addauthor{lt}{red}   
\addauthor{dc}{blue} 
\usepackage{mathtools}

\title{Fairness, Semi-Supervised Learning, and More:\\ A General Framework for Clustering with Stochastic Pairwise Constraints}
\author {
    Brian Brubach \textsuperscript{\rm 1},
    Darshan Chakrabarti \textsuperscript{\rm 2},
    John P.\ Dickerson \textsuperscript{\rm 3}, \\
    Aravind Srinivasan \textsuperscript{\rm 3},
    Leonidas Tsepenekas \textsuperscript{\rm 3}\\
}
\affiliations {
    \textsuperscript{\rm 1} Wellesley College,
    \textsuperscript{\rm 2} Carnegie Mellon University,
    \textsuperscript{\rm 3} University of Maryland, College Park \\
    bb100@wellesley.edu, darshanc@alumni.cmu.edu, \{john, srin, ltsepene\}@cs.umd.edu
}

\begin{document}

\maketitle

\begin{abstract}
  Metric clustering is fundamental in areas ranging from Combinatorial Optimization and Data Mining, to Machine Learning and Operations Research. However, in a variety of situations we may have additional requirements or knowledge, distinct from the underlying metric, regarding which pairs of points should be clustered together. To capture and analyze such scenarios, we introduce a novel family of \emph{stochastic pairwise constraints}, which we incorporate into several essential clustering objectives (radius/median/means). Moreover, we demonstrate that these constraints can succinctly model an intriguing collection of applications, including among others \emph{Individual Fairness} in clustering and \emph{Must-link} constraints in semi-supervised learning. Our main result consists of a general framework that yields approximation algorithms with provable guarantees for important clustering objectives, while at the same time producing solutions that respect the stochastic pairwise constraints. Furthermore, for certain objectives we devise improved results in the case of Must-link constraints, which are also the best possible from a theoretical perspective. Finally, we present experimental evidence that validates the effectiveness of our algorithms.

\end{abstract}

\section{Introduction}

In a generic metric clustering problem, there is a set of points $\mathcal{C}$, requiring service from a set of locations $\mathcal{F}$, where both $\mathcal{C}$ and $\mathcal{F}$ are embedded in some metric space. The sets $\mathcal{C}, \mathcal{F}$ do not need to be disjoint, and we may very well have $\mathcal{C} = \mathcal{F}$. The goal is then to choose a set of locations $S \subseteq \mathcal{F}$, where $S$ might have to satisfy additional problem-specific requirements and an assignment $\phi: \mathcal{C} \mapsto S$, such that a metric-related objective function over $\mathcal{C}$ is minimized. 

However, in a variety of situations there may be external and metric-independent constraints imposed on $\phi$, regarding which pairs of points $j,j' \in \mathcal{C}$ should be clustered together, i.e., constraints forcing a linkage $\phi(j) = \phi(j')$. In this work, we generalize this deterministic requirement, by introducing a novel family of \emph{stochastic pairwise constraints}. Our input is augmented with multiple sets $P_q$ of pairs of points from $\mathcal{C}$ ($P_q \subseteq \binom{\mathcal{C}}{2}$ for each $q$), and values $\psi_q \in [0,1]$. Given these, we ask for a randomized solution, which ensures that in expectation at most $\psi_q |P_q|$ pairs of $P_q$ are separated in the returned assignment. 
In Sections \ref{definitions}-\ref{Motivations}, we discuss how these constraints have enough expressive power to capture a wide range of applications such as extending the notion of \emph{Individual Fairness} from classification to clustering, and incorporating elements of Semi-Supervised clustering.

Another constraint we address is when $\mathcal{C} = \mathcal{F}$ and every chosen point $j \in S$ must serve as an exemplar of the cluster it defines (the set of all points assigned to it). The subtle difference here, is that an exemplar point should be assigned to its own cluster, i.e., $\phi(j) = j$ for all $j \in S$. This constraint is highly relevant in strict classification settings, and is trivially satisfied in vanilla clustering variants where each point is always assigned to its nearest point in $S$. However, the presence of additional requirements on $\phi$ makes its satisfaction more challenging. Previous literature, especially in the context of fairness in clustering \cite{anderson2020,esmaeili2020,bera2020,bercea2019}, does not address this issue, but in our framework we explicitly offer the choice of whether or not to enforce it.


\subsection{Formal Problem Definitions}\label{definitions}

We are given a set of points $\mathcal{C}$ and a set of locations $\mathcal{F}$, in a metric space characterized by the distance function $d: \mathcal{C} \cup \mathcal{F} \times \mathcal{C} \cup \mathcal{F} \mapsto \mathbb{R}_{\geq 0}$, which satisfies the triangle inequality. 
Moreover, the input includes a concise description of a set $\mathcal{L} \subseteq 2^{\mathcal{F}}$, that captures the allowable configurations of location openings. The goal of all problems we consider, is to find a set $S \subseteq \mathcal{F}$, with $S \in \mathcal{L}$, and an efficiently-sampleable distribution $\mathcal{D}$ over assignments $\mathcal{C} \mapsto S$, such that for a randomly drawn $\phi \sim \mathcal{D}$ we have: (i) an objective function being minimized, and (ii) depending on the variant at hand, further constraints are satisfied by $\phi$. We study two types of additional constraints imposed on $\phi$.
\begin{itemize}
    \item \emph{Stochastic Pairwise Constraints} (SPC): We are given a family of sets $\mathcal{P} = \{P_1, P_2, \ldots\}$, where each $P_q \subseteq \binom{\mathcal{C}}{2}$ is a set of pairs of points from $\mathcal{C}$, and a sequence $\psi = (\psi_1, \psi_2, \ldots)$ with $\psi_q \in [0, 1]$. We then want $\sum_{\{j,j'\} \in P_q}\Pr_{\phi \sim \mathcal{D}}[\phi(j) \neq \phi(j')] \leq \psi_q |P_q|, ~\forall P_q \in \mathcal{P}$.
    \item \emph{Centroid Constraint} (CC): When this is imposed on any of our problems, we must first have $\mathcal{C} = \mathcal{F}$. In addition, we should ensure that $\Pr_{\phi \sim \mathcal{D}}[\phi(i) = i]=1$ for all $i \in S$.
\end{itemize}
\textbf{Special Cases of SPC:} When each $P_q \in \mathcal{P}$ has $|P_q| = 1$, we get two interesting resulting variants.
\begin{itemize}
    \item $\psi_q = 0, \forall q$: For each $P_q = \big{\{}\{j,j'\}\big{\}}$ we must ensure that $j,j'$ have $\Pr_{\phi \sim \mathcal{D}}[\phi(j) = \phi(j')] = 1$, and hence we call such constraints \emph{must-link} (ML). Further, since there is no actual randomness involved in these constraints, we assume w.l.o.g. that $|\mathcal{D}| = 1$, and only solve for a single $\phi: \mathcal{C} \mapsto S$ instead of a distribution over assignments.
    \item $\psi_q \geq 0 , \forall q$: For each $P_q = \big{\{}\{j,j'\}\big{\}}$ we must have $\Pr_{\phi \sim \mathcal{D}}[\phi(j) \neq \phi(j')] \leq \psi_q$, and therefore we call this constraint \emph{probabilistic-bounded-separation} (PBS).
\end{itemize}
The objective functions we consider are:
\begin{itemize}
    \item \textbf{$\mathcal{L}$-center/$\mathcal{L}$-supplier:} Here we aim for the minimum $\tau$ (``radius''), such that $\Pr_{\phi \sim \mathcal{D}}[d(\phi(j),j) \leq \tau] = 1$ for all $j \in \mathcal{C}$. Further, in the $\mathcal{L}$-center setting, we have $\mathcal{C} = \mathcal{F}$.
    \item \textbf{$\mathcal{L}$-median ($p=1$)/$\mathcal{L}$-means ($p=2$):} Here the goal is to minimize $(\sum_{j \in \mathcal{C}}\mathbb{E}_{\phi \sim \mathcal{D}}[d(\phi(j),j)^p])^{1/p}$.
\end{itemize}

There are four types of location specific constraints that we study in this paper. In the first, which we call \emph{unrestricted}, $\mathcal{L} = 2^\mathcal{F}$ and hence any set of locations can serve our needs. In the second, we have $\mathcal{L} = \{S \subseteq \mathcal{F} ~|~ |S| \leq k\}$ for some given positive integer $k$. This variant gives rise to the popular \emph{$k$-center/$k$-supplier/$k$-median/$k$-means} objectives. In the third, we assume that each $i \in \mathcal{F}$ has an associated cost $w_i \geq 0$, and for some given $W \geq 0$ we have $\mathcal{L} = \{S \subseteq \mathcal{F}~|~ \sum_{i \in S}w_i \leq W\}$. In this case the resulting objectives are called \emph{knapsack-center/knapsack-supplier/knapsack-median/knapsack-means}. Finally, if the input also consists of a matroid $\mathcal{M}=(\mathcal{F}, \mathcal{I})$, where $\mathcal{I} \subseteq 2^\mathcal{F}$ the family of independent sets of $\mathcal{M}$, we have $\mathcal{L} = \mathcal{I}$, and the objectives are called \emph{matroid-center/matroid-supplier/matroid-median/matroid-means}.

To specify the problem at hand, we use the notation \textbf{Objective}-\textbf{List of Constraints}. For instance, \textbf{$\mathcal{L}$-means-SPC-CC} is the $\mathcal{L}$-means problem, where we additionally impose the SPC and CC constraint. We could also further specify $\mathcal{L}$, by writing for example \textbf{$k$-means-SPC-CC}. Moreover, observe that when no constraints on $\phi$ are imposed, we get the vanilla version of each objective, where the lack of any stochastic requirement implies that the distribution $\mathcal{D}$ once more has support of size $1$, i.e., $|\mathcal{D}|=1$, and we simply solve for just an assignment $\phi: \mathcal{C} \mapsto S$.

\subsection{Motivation}\label{Motivations}

In this section we present a wide variety of applications, that can be effectively modeled by our newly introduced SPCs.

\textbf{Fairness: }With machine-learning clustering approaches being ubiquitous in everyday decision making, a natural question that arises and has recently captured the interest of the research community, is how to avoid clusterings which perpetuate existing social biases. 

The \emph{individual} approach to fair classification introduced in the seminal work of~\cite{Dwork2012} assumes that we have access to an additional metric, separate from the feature space, which captures the true ``similarity'' between points (or some approximation of it). This similarity metric may be quite different from the feature space $d$ (e.g., due to redundant encodings of features such as race), and its ultimate purpose is to help ``treat similar candidates similarly''. Note now that the PBS constraint introduced earlier, can succinctly capture this notion. For two points $j,j'$, we may have $\psi_{j,j'} \in [0,1]$ as an estimate of their true similarity (with $0$ indicating absolute identity), and interpret unfair treatment as deterministically separating these two points in the final solution. Hence, a fair randomized approach would cluster $j$ and $j'$ apart with probability at most $\psi_{j,j'}$.

A recent work that explores individual fairness in clustering is \cite{anderson2020}. Using our notation, the authors in that paper require a set $S \in \mathcal{L}$, and for all $j \in \mathcal{C}$ a distribution $\phi_j$ that assigns $j$ to each $i \in S$ with probability $\phi_{i,j}$. Given that, they seek solutions that minimize the clustering objectives, while ensuring that for given pairs $j,j'$, their assignment distributions are statistically similar based on some metric $D$ that captures distributional proximity (e.g., total variation and KL-divergence). In other words, they interpret individual fairness as guaranteeing $D(\phi_j, \phi_{j'}) \leq p_{j,j'}$ for all provided pairs $\{j,j'\}$ and values $p_{j,j'}$. Although this work is interesting in terms of initiating the discussion on individual fair clustering, it has a significant modeling issue. To be more precise, suppose that for $j,j'$ the computed $\phi_j, \phi_{j'}$ are both the uniform distribution over $S$. Then according to that paper's definition a fair solution is achieved. However, the actual probability of placing $j,j'$ in different clusters (hence treating them unequally) is almost $1$ if we do not consider any correlation between $\phi_j$ and $\phi_{j'}$. On the other hand, our definition which instead asks for a distribution $\mathcal{D}$ over assignments $\phi: \mathcal{C} \mapsto S$, always provides meaningful results, since it bounds the quantity that really matters, i.e., the probability of separating $j$ and $j'$ in a random $\phi \sim \mathcal{D}$.

Another closely related work in the context of individual fair clustering is \cite{brubach2020}. The authors of that paper study a special case of PBS, where for each $j,j' \in \mathcal{C}$ we have $\psi_{j,j'} = d(j,j')/\tau^*$, with $\tau^*$ the objective value of the optimal solution. They then provide a $\log k$-approximation for the $k$-center objective under the above constraints. Compared to that, our framework 1) can handle the median and means objectives as well, 2) can incorporate further requirements on the set of chosen locations (unrestricted/knapsack/matroid), 3) allows for arbitrary values for the separation probabilities $\psi_{j,j'}$, and 4) provides smaller constant-factor approximations for the objective functions.

\textbf{Semi-Supervised Clustering: }A common example of ML constraints is in the area of semi-supervised learning~\cite{Wagstaff2001,Basu2008,Zhu2006}. There we assume that pairs of points have been annotated (e.g., by human experts) with additional information about their similarity~\cite{Zhang2007}, or that some points may be explicitly labeled~\cite{zhu2003,Bilenko2004integrating}, allowing pairwise relationships to be inferred. Then these extra requirements are incorporated in the algorithmic setting in the form of ML constraints. Further, our SPCs capture the scenario where the labeler generating the constraints is assumed to make some bounded number of errors (by associating each labeler with a set $P_q$ and an accuracy $\psi_q$), and also allow for multiple labelers (e.g., from crowdsourcing labels) with different accuracies. Similar settings have been studied by~\cite{Chang2017multiple,Luo2018semi} as well.

\textbf{OTU Clustering: }The field of metagenomics involves analyzing environmental samples of genetic material to explore the vast array of bacteria that cannot be analyzed through traditional culturing approaches. A common practice in the study of these microbial communities is the \emph{de novo} clustering of genetic sequences (e.g., 16S rRNA marker gene sequences) into Operational Taxonomic Units (OTUs)~\cite{Edgar2013,Westcott2017}, that ideally correspond to clusters of closely related organisms. One of the most ubiquitous approaches to this problem involves taking a fixed radius (e.g., $97\%$ similarity based on string alignment~\cite{Stackebrandt1994}) and outputting a set of center sequences, such that all points are assigned to a center within the given radius~\cite{Edgar2013,Ghodsi2011}. In this case, we do not know the number of clusters a priori, but we may be able to generate pairwise constraints based on a distance/similarity threshold as in~\cite{Westcott2017} or reference databases of known sequences. Thus, the ``unrestricted'' variant of our framework is appropriate here, where the number of clusters should be discovered, but radius and pairwise information is known or estimated. Other work in this area has considered \emph{conspecific probability}, a given probability that two different sequences belong to the same species (easily translated to PBS) and \emph{adverse triplets}; sets of ML constraints that cannot all be satisfied simultaneously (an appropriate scenario for a set $P_q$ as defined in Section \ref{definitions})\cite{Edgar2018updating}.

\textbf{Community Preservation: }There are scenarios where clustering a \emph{group}/\emph{community} of points together is beneficial for the coherence and quality of the final solution. Examples of this include assigning students to schools such that students living in the same neighborhood are not placed into different schools, vaccinating people with similar demographics in a community (e.g., during a pandemic), and drawing congressional districts with the intent to avoid the practice of gerrymandering. Given such a group of points $G$, we let $P_G = \binom{G}{2}$, and set a tolerance parameter $\psi_G \in [0,1]$. Then, our SPCs will make sure that in expectation at most $\psi_G |G|$ pairs from $G$ are separated, and thus a $(1-\psi_G)$ fraction of the community is guaranteed to be preserved. Finally, Markov's inequality also gives tail bounds on this degree of separation for all $G$.

\subsection{Our Contribution}

In Section \ref{sec:2} we present our main algorithmic result, which is based on the two-step approach of \cite{bercea2019,Chierichetti2017}. Unlike previous works utilizing this technique, the most serious technical difficulty we faced was not in the LP-rounding procedure, but rather in the formulation of an appropriate assignment-LP relaxation. Letting $P_\mathcal{L}$ be any problem in $\{\mathcal{L}$-$\text{center}, \mathcal{L}$-$\text{supplier}, \mathcal{L}$-$\text{median}, \mathcal{L}$-$\text{means}\}$ and $\mathcal{L}$ any of the four location settings, we get:
\begin{theorem}\label{intr-thm1}
Let $\tau^*$ the optimal value of a \textbf{$P_{\mathcal{L}}$-SPC} instance, and $\rho$ the best approximation ratio for \textbf{$P_{\mathcal{L}}$}. Then our algorithm chooses a set $S_{P_{\mathcal{L}}}$ and constructs an appropriate distribution over assignments $\mathcal{D}$, such that $S_{P_{\mathcal{L}}} \in \mathcal{L}$, $\sum_{\{j,j'\} \in P_q}\Pr_{\phi \sim \mathcal{D}}[\phi(j) \neq \phi(j')] \leq 2\psi_q |P_q| ~\forall P_q \in \mathcal{P}$, and
\begin{enumerate}
    \item $P_{\mathcal{L}}$ is  $\mathcal{L}$-center$(\alpha=1)$/$\mathcal{L}$-supplier$(\alpha=2)$: Here we get $\Pr_{\phi \sim \mathcal{D}}[d(\phi(j), j) \leq (\alpha+\rho)\tau^*] = 1$, for all $j \in \mathcal{C}$.
    \item $P_{\mathcal{L}}$ is $\mathcal{L}$-median$(p=1)$/$\mathcal{L}$-means$(p=2)$: Here we get $(\sum_{j \in \mathcal{C}}\mathbb{E}_{\phi\sim \mathcal{D}}[d(\phi(j),j)^p])^{1/p} \leq (2+\rho)\tau^*$.
\end{enumerate}
Finally, sampling a $\phi \sim \mathcal{D}$ can be done in polynomial time.
\end{theorem}

Given that the value $\rho$ is a small constant for all variations of $P_{\mathcal{L}}$ that we consider, we see that our algorithmic framework gives indeed good near-optimal guarantees. Moreover, a tighter analysis when $\mathcal{L}=2^\mathcal{F}$ yields the next result.
\begin{theorem}
When $\mathcal{L}=2^\mathcal{F}$, our algorithm has the same guarantees as those in Theorem \ref{intr-thm1}, but this time the cost of the returned solution for \textbf{$P_{unrestricted}$-SPC} is at most $\tau^*$.
\end{theorem}

Although imposing no constraint on the set of chosen locations yields trivial problems in vanilla settings, the presence of SPCs makes even this variant NP-hard. Specifically, we show the following theorem in Appendix \ref{appendix}.

\begin{theorem}\label{np-hard}
The problem \textbf{unrestricted-$\mathcal{O}$-SPC} is NP-hard, where $\mathcal{O} \in \{\text{center, supplier, median, means}\}$.
\end{theorem}

In Section \ref{sec:3} we consider settings where each $j \in S$ must serve as an exemplar of its defining cluster. Hence, we incorporate the Centroid constraint in our problems. As mentioned earlier, previous work in the area of fair clustering had ignored this issue. Our first result follows.
\begin{theorem}\label{intr-thm3}
Let $\tau^*$ the optimal value of a \textbf{$k$-center-SPC-CC} instance. Then our algorithm chooses $S_k \subseteq \mathcal{C}$ and constructs a distribution $\mathcal{D}$, such that sampling $\phi \sim \mathcal{D}$ can be done efficiently, $|S_{k}| \leq k$,  $\sum_{\{j,j'\} \in P_q}\Pr_{\phi \sim \mathcal{D}}[\phi(j) \neq \phi(j')] \leq 2\psi_q |P_q| ~\forall P_q \in \mathcal{P}$, $\Pr_{\phi \sim \mathcal{D}}[d(\phi(j), j) \leq 3\tau^*] = 1$ for all $j \in \mathcal{C}$, and $\Pr_{\phi \sim \mathcal{D}}[\phi(i)=i] = 1$ for all $i \in S_{k}$.
\end{theorem}
To address all objective functions under the Centroid constraint, we demonstrate (again in Section \ref{sec:3}) a reassignment procedure that gives the following result.
\begin{theorem}\label{intr-thm4}
Let $\lambda$ the approximation ratio for the objective of \textbf{$P_\mathcal{L}$-SPC} achieved in Theorem \ref{intr-thm1}. Then, our reassignment procedure applied to the solution produced by the algorithm mentioned in Theorem \ref{intr-thm1}, gives an approximation ratio $2\lambda$ for \textbf{$P_\mathcal{L}$-SPC-CC}, while also preserving the SPC guarantees of Theorem \ref{intr-thm1} and satisfying the CC, when $\mathcal{L} = 2^\mathcal{C}$ or $\mathcal{L} = \{S' \subseteq \mathcal{C}:~ |S'| \leq k\}$ for some given positive integer $k$.
\end{theorem}

As for ML constraints, since they are a special case of SPCs, our results for the latter also address the former. However, in Section \ref{sec:4} we provide improved approximation algorithms for a variety of problem settings with ML constraints. Our main result is summarized in the following theorem. 
\begin{theorem}\label{intr-thm5}
There exists a $2/3/3/3$-approximation algorithm for \textbf{$k$-center-ML}/\textbf{knapsack-center-ML}/\textbf{$k$-supplier-ML}/\textbf{knapsack-supplier-ML}. This algorithm is also the best possible in terms of the approximation ratio, unless $P=NP$. In addition, it satisfies without any further modifications the Centroid constraint.
\end{theorem}
Although ML constraints have been extensively studied in the semi-supervised literature \cite{Basu2008}, to the extent of our knowledge we are the first to tackle them purely from a Combinatorial Optimization perspective, with the exception of \cite{Davidson2010}. This paper provides a $(1+\epsilon)$ approximation for $k$-center-ML, but only in the restricted $k=2$ setting. 

\subsection{Further Related Work}
Clustering problems have been a longstanding area of research in Combinatorial Optimization, with all important settings being thoroughly studied \cite{Hochbaum1986,Gonzalez1985,harris2017,byrka2017,ola2017,Chakra2019}.

The work that initiated the study of fairness in clustering is \cite{Chierichetti2017}. That paper addresses a notion of demographic fairness, where points are given a certain color indicating some protected attribute, and then the goal is to compute a solution that enforces a fair representation of each color in every cluster. Further work on similar notions of demographic fairness includes \cite{bercea2019,bera2020,esmaeili2020,huang2019,backurs2019,ahmadian2019}.

Finally, a separation constraint similar to PBS is found in \cite{Davidson2010}. In that paper however, the separation is deterministic and also depends on the underlying distance between two points. Due to their stochastic nature, our PBS constraints allow room for more flexible solutions, and also capture more general separation scenarios, since the $\psi_p$ values can be arbitrarily chosen.

\subsection{An LP-Rounding Subroutine}
We present an important subroutine developed by \cite{Kleinberg2002}, which we repeatedly use in our results, and call it \textbf{KT-Round}. Suppose we have a set of elements $V$, a set of labels $L$, and a set of pairs $E \subseteq \binom{V}{2}$. Consider the following Linear Program (LP).
\begin{align}
    &\displaystyle\sum_{l \in L}x_{l,v} = 1, ~&\forall v \in V&  \label{KT-1} \\
&z_{e,l} \geq x_{l,v} - x_{l,w}, ~&\forall e = \{v,w\} \in E, ~\forall l \in L& \label{KT-2} \\
&z_{e,l} \geq x_{l,w} - x_{l,v}, ~&\forall e = \{v,w\} \in E, ~\forall l \in L& \label{KT-3} \\
&z_e = \frac{1}{2}\displaystyle\sum_{l \in L}z_{e,l}, ~&\forall e = \{v,w\} \in E& \label{KT-4} \\
&0 \leq x_{l,v}, z_e, z_{e,l} \leq 1, ~&\forall v \in V, \forall e \in E, \forall l \in L& \label{KT-5}
\end{align}

\begin{theorem}{\cite{Kleinberg2002}}\label{round}
Given a feasible solution $(x, z)$ of (\ref{KT-1})-(\ref{KT-5}), there exists a randomized rounding approach \textbf{KT-Round}($V,L,E,x,z$), which in polynomial expected time assigns each $v \in V$ to a $\phi(v) \in L$, such that:
\begin{enumerate}
    \item $\Pr[\phi(v) \neq \phi(w)] \leq 2z_e, ~~\forall e = \{v,w\} \in E$
    \item $\Pr[\phi(v) = l] = x_{l,v}, ~~~~\forall v \in V, ~\forall l \in L$
\end{enumerate}
\end{theorem}
\section{A General Framework for Approximating Clustering Problems with SPCs}\label{sec:2}

In this section we show how to achieve approximation algorithms with provable guarantees for \textbf{$\mathcal{L}$-center-SPC/$\mathcal{L}$-supplier-SPC/$\mathcal{L}$-median-SPC/$\mathcal{L}$-means-SPC} using a general two-step framework. At first, let $P_{\mathcal{L}}$ denote any of the vanilla versions of the objective functions we consider, i.e., $P_{\mathcal{L}} \in \{\mathcal{L}$-$\text{center}, \mathcal{L}$-$\text{supplier}, \mathcal{L}$-$\text{median}, \mathcal{L}$-$\text{means}\}$. 

To tackle a $P_{\mathcal{L}}$-SPC instance, we begin by using on it any known $\rho$-approximation algorithm $A_{P_{\mathcal{L}}}$ for $P_{\mathcal{L}}$. This gives a set of locations $S_{P_{\mathcal{L}}}$ and an assignment $\phi_{P_{\mathcal{L}}}$, which yield an objective function cost of $\tau_{P_{\mathcal{L}}}$ for the corresponding $P_{\mathcal{L}}$ instance. In other words, we drop the SPC constraints from the $P_{\mathcal{L}}$-SPC instance, and simply treat it as its vanilla counterpart. Although $\phi_{P_{\mathcal{L}}}$ may not satisfy the SPCs, we are going to use the set $S_{P_{\mathcal{L}}}$ as our chosen locations. The second step in our framework would then consist of constructing the appropriate distribution over assignments. Toward that end, consider the following LP, where $P' = \displaystyle\cup_{P_q \in \mathcal{P}}P_q$.
\begin{align}
& \sum_{i \in S_{P_{\mathcal{L}}}}x_{i,j} = 1  &\forall j \in \mathcal{C}& \label{s2-LP-1} \\
&z_{e,i} \geq x_{i,j} - x_{i,j'} &\forall e = \{j,j'\} \in P', ~\forall i \in S_{P_{\mathcal{L}}}& \label{s2-LP-2}\\
&z_{e,i} \geq x_{i,j'} - x_{i,j} &\forall e = \{j,j'\} \in P', ~\forall i \in S_{P_{\mathcal{L}}}& \label{s2-LP-3}\\
&z_e = \frac{1}{2}\sum_{i \in S_{P_{\mathcal{L}}}}z_{e,i} &\forall e \in P'& \label{s2-LP-4}\\
&\sum_{e \in P_q} z_e \leq \psi_q |P_q| &\forall P_q \in \mathcal{P}& \label{s2-LP-5}\\
&0 \leq x_{i,j}, z_e, z_{e,i} \leq 1 &\forall i \in S_{P_{\mathcal{L}}}, \forall j \in \mathcal{C}, \forall e \in P'& \label{s2-LP-6}
\end{align}

The variable $x_{i,j}$ can be interpreted as the probability of assigning point $j$ to location $i \in S_{P_{\mathcal{L}}}$. To understand the meaning of the $z$ variables, it is easier to think of the integral setting, where $x_{i,j} = 1$ iff $j$ is assigned to $i$ and $0$ otherwise. In this case, $z_{e,i}$ is $1$ for $e = \{j,j'\}$ iff exactly one of $j$ and $j'$ are assigned to $i$. Thus, $z_e$ is $1$ iff $j$ and $j'$ are separated. We will later show that in the fractional setting $z_e$ is a lower bound on the probability that $j$ and $j'$ are separated. Therefore, constraint (\ref{s2-LP-1}) simply states that every point must be assigned to a center, and given the previous discussion, (\ref{s2-LP-5}) expresses the provided SPCs.

Depending on which exact objective function we optimize, we must augment LP (\ref{s2-LP-1})-(\ref{s2-LP-6}) accordingly.
\begin{itemize}
    \item \textbf{$\mathcal{L}$-center ($\alpha=1$)/$\mathcal{L}$-supplier ($\alpha=2$):} Here we assume w.l.o.g. that the optimal radius $\tau^*_{SPC}$ of the original $P_{\mathcal{L}}$\textbf{-SPC} instance is known. Observe that this value is always the distance between some point and some location, and hence there are only polynomially many alternatives for it. Thus, we execute our algorithm for each of those, and in the end keep the outcome that resulted in a feasible solution of minimum value. Given now $\tau^*_{SPC}$, we add the following constraint to the LP. 
    \begin{align}
        x_{i,j} = 0, ~\forall i,j: ~d(i,j) > \tau_{P_{\mathcal{L}}} + \alpha \cdot \tau^*_{SPC}\label{s2-LP-cnt}
    \end{align}
    \item \textbf{$\mathcal{L}$-median ($p=1$)/$\mathcal{L}$-means ($p=2$):} In this case, we augment the LP with the following objective function.
    \begin{align}
        \min ~ \sum_{j \in \mathcal{C}}\sum_{i \in S_{P_{\mathcal{L}}}}x_{i,j} \cdot d(i,j)^p \label{s2-LP-md}
    \end{align}
\end{itemize}

The second step of our framework begins by solving the appropriate LP for each variant of $P_{\mathcal{L}}$, in order to acquire a fractional solution $(\bar x, \bar z)$ to that LP. Finally, the distribution $\mathcal{D}$ over assignments $\mathcal{C} \mapsto S_{P_{\mathcal{L}}}$ is constructed by running \textbf{KT-Round}($\mathcal{C}, S_{P_{\mathcal{L}}}, P',\bar x, \bar z$). Notice that this will yield an assignment $\phi \sim \mathcal{D}$, where $\mathcal{D}$ results from the internal randomness of \textbf{KT-Round}. Our overall approach for solving \textbf{$P_{\mathcal{L}}$-SPC} is presented in Algorithm \ref{alg-1}.

\begin{algorithm}[t]
$(S_{P_{\mathcal{L}}}, \phi_{P_{\mathcal{L}}}) \gets A_{P_{\mathcal{L}}}(\mathcal{C}, \mathcal{F}, \mathcal{L})$\;
Solve LP (\ref{s2-LP-1})-(\ref{s2-LP-6}) with (\ref{s2-LP-cnt}) for \textbf{$\mathcal{L}$-center/$\mathcal{L}$-supplier}, and with (\ref{s2-LP-md}) for \textbf{$\mathcal{L}$-median/$\mathcal{L}$-means}, and get a fractional solution $(\bar x, \bar z)$\;
$\phi \gets$ \textbf{KT-Round}$(\mathcal{C}, S_{P_{\mathcal{L}}}, P', \bar x, \bar z)$\;
\caption{Approximating  \textbf{$P_{\mathcal{L}}$-SPC}}\label{alg-1}
\end{algorithm}

\begin{theorem}\label{s2-thm1}
Let $\tau^*_{SPC}$ the optimal value of the given \textbf{$P_{\mathcal{L}}$-SPC} instance. Then Algorithm \ref{alg-1} guarantees that $S_{P_{\mathcal{L}}} \in \mathcal{L}$, $\sum_{\{j,j'\} \in P_q}\Pr_{\phi \sim \mathcal{D}}[\phi(j) \neq \phi(j')] \leq 2\psi_q |P_q| ~\forall P_q \in \mathcal{P}$ and
\begin{enumerate}
    \item $P_{\mathcal{L}}$ is  $\mathcal{L}$-center$(\alpha=1)$/$\mathcal{L}$-supplier$(\alpha=2)$: Here we get $\Pr_{\phi \sim \mathcal{D}}[d(\phi(j), j) \leq \alpha\cdot \tau^*_{SPC} + \tau_{P_{\mathcal{L}}}] = 1$, for all $j \in \mathcal{C}$.
    \item $P_{\mathcal{L}}$ is $\mathcal{L}$-median$(p=1)$/$\mathcal{L}$-means$(p=2)$: Here we get $(\sum_{j \in \mathcal{C}}\mathbb{E}_{\phi\sim \mathcal{D}}[d(\phi(j),j)^p])^{1/p} \leq 2\tau^*_{SPC} + \tau_{P_{\mathcal{L}}}$.
\end{enumerate}
\end{theorem}

\begin{proof}
At first, notice that since $S_{P_{\mathcal{L}}}$ results from running $A_{P_{\mathcal{L}}}$, it must be the case that $S_{P_{\mathcal{L}}} \in \mathcal{L}$.

Focus now on LP (\ref{s2-LP-1})-(\ref{s2-LP-6}) with either (\ref{s2-LP-cnt}) or (\ref{s2-LP-md}), depending on the underlying objective. In addition, let $S^* \in \mathcal{L}$ and $\mathcal{D}^*$ be the set of locations and the distribution over assignments $\mathcal{C} \mapsto S^*$, that constitute the optimal solution of \textbf{$P_{\mathcal{L}}$-SPC}. Given those, let $x^*_{i,j} = \Pr_{\phi \sim \mathcal{D}^*}[\phi(j) = i]$ for all $i \in S^*$ and all $j \in \mathcal{C}$. Moreover, for every $i' \in S^*$ let $\kappa(i') = \argmin_{i \in S_{P_{\mathcal{L}}}}d(i,i')$ by breaking ties arbitrarily. Finally, for all $i \in S_{P_{\mathcal{L}}}$ we define $N(i) = \{i' \in S^* ~|~ i = \kappa(i')\}$, and notice that the sets $N(i)$ form a partition of $S^*$.

Consider now the vectors $\hat{x}_{i,j} = \sum_{i' \in N(i)}x^*_{i',j}$ for every $i \in S_{P_{\mathcal{L}}}$ and $j \in \mathcal{C}$, and $\hat{z}_{e,i} = |\hat{x}_{i,j} - \hat{x}_{i,j'}|$ for every $e = \{j,j'\} \in P'$ and $i \in S_{P_{\mathcal{L}}}$. We first show that the above vectors constitute a feasible solution of LP (\ref{s2-LP-1})-(\ref{s2-LP-6}). Initially, notice that constraints (\ref{s2-LP-2}), (\ref{s2-LP-3}), (\ref{s2-LP-4}), (\ref{s2-LP-6}) are trivially satisfied. Regarding constraint (\ref{s2-LP-1}), for any $j \in \mathcal{C}$ we have:
\begin{align}
    \sum_{i \in S_{P_{\mathcal{L}}}}\hat{x}_{i,j} = \sum_{i \in S_{P_{\mathcal{L}}}}\sum_{i' \in N(i)}x^*_{i',j} = \sum_{i' \in S^*}x^*_{i',j} = 1 \notag
\end{align}
The second equality follows because the sets $N(i)$ induce a partition of $S^*$. The last equality is due to the optimal solution $\mathcal{D}^*, S^*$ satisfying $\sum_{i \in S^*}\Pr_{\phi \sim \mathcal{D}^*}[\phi(j) = i] = 1$.

To show satisfaction of constraint (\ref{s2-LP-5}) focus on any $e = \{j,j'\} \in P'$ and $i \in S_{P_{\mathcal{L}}}$. We then have:
\begin{align}
    \hat{z}_{e,i} = \Big{|}\sum_{i' \in N(i)}(x^*_{i',j} - x^*_{i',j'})\Big{|} \leq \sum_{i' \in N(i)}| x^*_{i',j} - x^*_{i',j'} |\notag
\end{align}
Therefore, we can easily upper bound $\hat{z}_e$ as follows:
\begin{align}
\hat{z}_e &= \frac{1}{2}\sum_{i \in S_{P_{\mathcal{L}}}}\hat{z}_{e,i} \leq \frac{1}{2}\sum_{i \in S_{P_{\mathcal{L}}}}\sum_{i' \in N(i)}| x^*_{i',j} - x^*_{i',j} | \notag \\
     &\leq \frac{1}{2}\displaystyle\sum_{i' \in S^*}| x^*_{i',j} - x^*_{i',j} | \label{z-bound}
\end{align}
To move one, notice that:
\begin{align}
\Pr_{\phi \sim \mathcal{D}^*}[\phi(j) = \phi(j')] &= \sum_{i' \in S^*}\Pr_{\phi \sim \mathcal{D}^*}[\phi(j) = i' ~ \wedge ~ \phi(j') = i'] \notag \\
&\leq \sum_{i' \in S^*} \min\{ x^*_{i',j}, ~ x^*_{i',j'}\} \label{z-bound-2}
\end{align}
To relate (\ref{z-bound}) and (\ref{z-bound-2}) consider the following trick.
\begin{align}
&\sum_{i' \in S^*} \min\{ x^*_{i',j}, ~ x^*_{i',j'}\} + \frac{1}{2} \sum_{i' \in S^*}| x^*_{i',j} -  x^*_{i',j'}| = \notag \\ 
&\sum_{i' \in S^*} \Big{(}\min\{  x^*_{i',j}, ~ x^*_{i',j'}\} + \frac{| x^*_{i',j} -  x^*_{i',j'}|}{2}\Big{)} = \notag \\
&\sum_{i' \in S^*} \frac{  x^*_{i',j} +  x^*_{i',j'}}{2} = \frac{2}{2} = 1 \label{z-bound-3}
\end{align}
Finally, combining (\ref{z-bound}),(\ref{z-bound-2}),(\ref{z-bound-3}) we get:
\begin{align}
\hat{z}_e \leq 1 - \displaystyle\sum_{i' \in S^*} \min\{  x^*_{i',j}, ~x^*_{i',j'}\}  \leq \Pr_{\phi \sim \mathcal{D}^*}[\phi(j) \neq \phi(j')]  \notag
\end{align}
Given the above, for every $P_q \in \mathcal{P}$ we have:
\begin{align}
\sum_{e \in P_q}\hat{z}_e \leq \sum_{\{j,j'\} \in P_q}\Pr_{\phi \sim \mathcal{D}^*}[\phi(j) \neq \phi(j')] \leq \psi_q |P_q| \notag
\end{align}
where the last inequality follows from optimality of $\mathcal{D}^*$.

Now that we know that $(\hat x, \hat z)$ is a feasible solution for (\ref{s2-LP-1})-(\ref{s2-LP-6}), we proceed by considering how this solution affects the objective function of each underlying problem.

\textbf{$\mathcal{L}$-center$(\alpha=1)$/$\mathcal{L}$-supplier$(\alpha=2)$:} The objective here is captured by the additional constraint (\ref{s2-LP-cnt}). Hence, we also need to show that $\hat x$ satisfies (\ref{s2-LP-cnt}), i.e., that for all $i \in S_{P_{\mathcal{L}}}$, $j \in \mathcal{C}$ for which $d(i,j) > \tau_{P_{\mathcal{L}}} + \alpha \cdot  \tau^*_{SPC}$, we have $\hat{x}_{i,j} = 0$. 

Suppose for the sake of contradiction that there exists a $j \in \mathcal{C}$ and an $i \in S_{P_{\mathcal{L}}}$ such that $d(i,j) > \tau_{P_{\mathcal{L}}} + \alpha \cdot \tau^*_{SPC}$ and $\hat{x}_{i,j} > 0$. Since $\hat{x}_{i,j} = \sum_{i' \in N(i)}x^*_{i',j}$, this implies that there exists $i' \in N(i)$ with $x^*_{i',j} > 0$, which consecutively implies $d(i',j) \leq \tau^*_{SPC}$. By the triangle inequality we get: 
\begin{align}
    d(i,j) \leq d(i,i') + d(i',j) \leq d(i,i') + \tau^*_{SPC} \label{s2-aux1}
\end{align}
In $\mathcal{L}$-center, because $i' \in N(i)$ we also have $d(i,i') \leq \tau_{P_{\mathcal{L}}}$, and so we reach a contradiction. In $\mathcal{L}$-supplier we have:
\begin{align}
    d(i,i') &\leq d(i', \phi_{P_{\mathcal{L}}}(j)) \leq d(i',j) + d(\phi_{P_{\mathcal{L}}}(j), j) \notag \\ &\leq \tau^*_{SPC} + \tau_{P_{\mathcal{L}}} \label{unrestr}
\end{align}
Combining (\ref{s2-aux1}),(\ref{unrestr}) gives the desired contradiction for the case of $\mathcal{L}$-supplier as well.

\textbf{$\mathcal{L}$-median$(p=1)$/$\mathcal{L}$-means$(p=2)$:} Here the overall objective function for $\hat x$ is given by:
\begin{align}
    &\Big{(} \sum_{j \in \mathcal{C}}\sum_{i \in S_{P_{\mathcal{L}}}}\hat{x}_{i,j}d(i,j)^p\Big{)}^{\frac{1}{p}} = \notag \\
    &\Big{(} \sum_{j \in \mathcal{C}}\sum_{i \in S_{P_{\mathcal{L}}}}\sum_{i' \in N(i)}x^*_{i',j}d(i,j)^p\Big{)}^{\frac{1}{p}} \label{s2-aux2}
\end{align}
In addition, for $i' \in N(i)$ we also get:
\begin{align}
    d(i,j) &\leq d(i,i') + d(i',j) \leq d(i', \phi_{P_{\mathcal{L}}}(j)) + d(i',j) \notag \\
    &\leq d(i',j) + d(\phi_{P_{\mathcal{L}}}(j), j) + d(i',j) \notag \\ &\leq 2d(i',j) + d(\phi_{P_{\mathcal{L}}}(j), j) \label{s2-aux3}
\end{align}
Combining (\ref{s2-aux2}), (\ref{s2-aux3}) and the fact that the median and means objectives are monotone norms, we get (\ref{s2-aux2}) $\leq A + B$, where:
\begin{align}
    A &= \Big{(} \sum_{j \in \mathcal{C}}\sum_{i \in S_{P_{\mathcal{L}}}}\sum_{i' \in N(i)}2x^*_{i',j}d(i',j)^p\Big{)}^{\frac{1}{p}} \notag \\
    &= 2\Big{(} \sum_{j \in \mathcal{C}}\sum_{i' \in S^*}x^*_{i',j}d(i',j)^p\Big{)}^{\frac{1}{p}} \leq 2\tau^*_{SPC} \label{s2-auxA} \\
    B &=  \Big{(} \sum_{j \in \mathcal{C}}\sum_{i \in S_{P_{\mathcal{L}}}}\sum_{i' \in N(i)}x^*_{i',j}d(\phi_{P_{\mathcal{L}}}(j),j)^p\Big{)}^{\frac{1}{p}} \notag \\
    &=\Big{(} \sum_{j \in \mathcal{C}}d(\phi_{P_{\mathcal{L}}}(j),j)^p\sum_{i \in S_{P_{\mathcal{L}}}}\sum_{i' \in N(i)}x^*_{i',j}\Big{)}^{\frac{1}{p}} \notag \\
    &=\Big{(} \sum_{j \in \mathcal{C}}d(\phi_{P_{\mathcal{L}}}(j),j)^p\sum_{i' \in S^*}x^*_{i',j}\Big{)}^{\frac{1}{p}} \notag \\
    &=\Big{(} \sum_{j \in \mathcal{C}}d(\phi_{P_{\mathcal{L}}}(j),j)^p\Big{)}^{\frac{1}{p}} = \tau_{P_{\mathcal{L}}}\label{s2-auxB}
\end{align}
Combining (\ref{s2-auxA}), (\ref{s2-auxB}) we finally get (\ref{s2-aux2}) $\leq 2\tau^*_{SPC} + \tau_{P_{\mathcal{L}}}$.

Since $(\hat x, \hat z)$ is a feasible solution to the appropriate version of the assignment LP, step 2 of Algorithm \ref{alg-1} is well-defined, and thus can compute a solution $(\bar x, \bar z)$ that satisfies (\ref{s2-LP-1})-(\ref{s2-LP-6}) and additionally: \textbf{i)} \begin{align}
    \bar{x}_{i,j} = 0, ~\forall i,j: ~d(i,j) > \tau_{P_{\mathcal{L}}} + \alpha \cdot \tau^*_{SPC} \label{s2-bar-cnt}
\end{align} 
for $\mathcal{L}$-center$(\alpha=1)$/$\mathcal{L}$-supplier$(\alpha=2)$, and \textbf{ii)} 
\begin{align}
    \Big{(} \sum_{j \in \mathcal{C}}\sum_{i \in S_{P_{\mathcal{L}}}}\bar{x}_{i,j}d(i,j)^p\Big{)}^{\frac{1}{p}} &\leq \Big{(} \sum_{j \in \mathcal{C}}\sum_{i \in S_{P_{\mathcal{L}}}}\hat{x}_{i,j}d(i,j)^p\Big{)}^{\frac{1}{p}} \notag \\
    &\leq 2\tau^*_{SPC} + \tau_{P_{\mathcal{L}}} \label{s2-bar-md}
\end{align}
for $\mathcal{L}$-median$(p=1)$/$\mathcal{L}$-means$(p=2)$.

Because $(\bar x, \bar z)$ satisfies (\ref{s2-LP-1}), (\ref{s2-LP-2}), (\ref{s2-LP-3}), (\ref{s2-LP-4}), (\ref{s2-LP-6}), \textbf{KT-Round} can be applied for $V = \mathcal{C}$, $L = S_{P_{\mathcal{L}}}$, $E = P'$. Let $\phi$ be the assignment returned by \textbf{KT-Round}$(\mathcal{C}, S_{P_{\mathcal{L}}}, P', \bar x, \bar z)$, and $\mathcal{D}$ the distribution representing the internal randomness of this process. From Theorem \ref{round} we have $\Pr_{\phi \sim \mathcal{D}}[\phi(j) \neq \phi(j')] \leq 2\bar{z}_e$, $\forall e = \{j,j'\} \in P'$. Hence, for every $P_q \in \mathcal{P}$:
\begin{align}
    \sum_{\{j,j'\} \in P_q}\Pr_{\phi \sim \mathcal{D}}[\phi(j) \neq \phi(j')] \leq 2\sum_{e \in P_q}\bar{z}_e \leq 2\psi_q |P_q| \notag
\end{align}
The last inequality is due to $\bar z$ satisfying (\ref{s2-LP-5}).

Regarding all the different objective functions, we have the following. For $\mathcal{L}$-center$(\alpha=1)$/$\mathcal{L}$-supplier$(\alpha=2)$, because of (\ref{s2-bar-cnt}) and the second property of Theorem \ref{round}, we know that a point $j \in \mathcal{C}$ will never be assigned to a location $i \in S_{P_{\mathcal{L}}}$, such that $d(i,j) > \tau_{P_{\mathcal{L}}} + \alpha \cdot \tau^*_{SPC}$. Therefore, $\Pr_{\phi \sim \mathcal{D}}[d(\phi(j), j) \leq \tau_{P_{\mathcal{L}}} + \alpha \cdot \tau^*_{SPC}] = 1$ for all $j \in \mathcal{C}$. As for the $\mathcal{L}$-median$(p=1)$/$\mathcal{L}$-means$(p=2)$ objectives, the second property of Theorem \ref{round} ensures:
\begin{align}
    &\Big{(}\sum_{j \in \mathcal{C}}\mathbb{E}_{\phi\sim \mathcal{D}}[d(\phi(j),j)^p]\Big{)}^{1/p} = \notag \\
    &\Big{(}\sum_{j \in \mathcal{C}}\sum_{i \in S_{P_{\mathcal{L}}}}\Pr_{\phi \sim \mathcal{D}}[\phi(j) = i]\cdot d(i,j)^p\Big{)}^{1/p} = \notag \\
    &\Big{(}\sum_{j \in \mathcal{C}}\sum_{i \in S_{P_{\mathcal{L}}}}\bar{x}_{i,j}\cdot d(i,j)^p\Big{)}^{1/p} \leq 2\tau^*_{SPC} + \tau_{P_{\mathcal{L}}} \notag
\end{align}
where the last inequality follows from (\ref{s2-bar-md}).
\end{proof}

Since $P_{\mathcal{L}}$ is a less restricted version of $P_{\mathcal{L}}$-SPC, the optimal solution value $\tau^*_{P_{\mathcal{L}}}$   for $P_{\mathcal{L}}$ in the original instance where we dropped the SPCs, should satisfy $\tau^*_{P_{\mathcal{L}}} \leq \tau^*_{SPC}$. Therefore, because $A_{P_{\mathcal{L}}}$ is a $\rho$-approximation algorithm for $P_{\mathcal{L}}$, we get $\tau_{P_{\mathcal{L}}} \leq \rho\cdot \tau^*_{SPC}$. The latter implies the following.

\begin{corollary}
The approximation ratio achieved through Algorithm \ref{alg-1} is $(\rho+1)$ for $\mathcal{L}$-center-SPC, and $(\rho+2)$ for $\mathcal{L}$-supplier-SPC/$\mathcal{L}$-median-SPC/$\mathcal{L}$-means-SPC.
\end{corollary}

\noindent \textbf{Tighter analysis for the} \emph{unrestricted} \textbf{($\mathcal{L} = 2^\mathcal{F}$) case:} For this case, a more careful analysis leads to the following. 

\begin{theorem}
When $\mathcal{L} = 2^\mathcal{F}$, Algorithm \ref{alg-1} achieves an objective value of at most $\tau^*_{SPC}$ for all objectives we study (center/supplier/median/means).
\end{theorem}

\begin{proof}
Note that the first step of our framework will choose $S_{P_{\mathcal{L}}} = \mathcal{F}$, and hence $\tau_{P_{\mathcal{L}}} = 0$. In addition, a closer examination of our analysis for the supplier/median/means settings reveals the following. Focus on (\ref{s2-aux1}) and (\ref{s2-aux3}), and note that because $S_{P_{\mathcal{L}}} = \mathcal{F}$, we get $d(i,i') = 0$ since $\kappa(i') = i'$ for each $ i' \in S^*$. Hence, $d(i,j) \leq d(i',j)$ which leads to an objective function value of at most $\tau^*_{SPC}$ for all problems.
\end{proof}
\section{Addressing the Centroid Constraint}\label{sec:3}

In this section we present results that incorporate the Centroid Constraint (CC) to a variety of the settings we study. Moreover, recall that for this case $\mathcal{C} =\mathcal{F}$, and hence the supplier objective reduces to the center one.

\subsection{Approximating $k$-center-SPC-CC}

Our approach for solving this problem heavily relies on Algorithm \ref{alg-1} with two major differences.

The first difference compared to Algorithm \ref{alg-1} lies in the approximation algorithm $A_{k}$ used to tackle $k$-center. For $k$-center there exists a $2$-approximation which given a target radius $\tau$, it either returns a solution where each $j \in \mathcal{C}$ gets assigned to a location $i_j$ with $d(i_j, j) \leq 2\tau$, or outputs an ``infeasible'' message, indicating that there exists no solution of radius $\tau$ (\cite{Hochbaum1986}).

Recall now that w.l.o.g. the optimal radius $\tau^{*}_{C}$ for the $k$-center-SPC-CC instance is known. In the first step of our framework we will use the variant of $A_{k}$ mentioned earlier with $\tau^{*}_{C}$ as its target radius, and get a set of chosen locations $S_k$. The second step is then the same as in Algorithm \ref{alg-1}, with the addition of the next constraint to the assignment LP:
\begin{align}
    x_{i,i} = 1, ~\forall i \in S_{k} \label{s3-cnt-centr}
\end{align}
The overall process is presented in Algorithm \ref{alg-2}.

\begin{algorithm}[t]
$(S_{k}, \phi_{k}) \gets A_{k}(\mathcal{C}, \mathcal{F}, \mathcal{L}, \tau^{*}_{C})$\;
Solve LP (\ref{s2-LP-1})-(\ref{s2-LP-6}) with (\ref{s2-LP-cnt}), (\ref{s3-cnt-centr}) and $S_k$ as the chosen locations, and get a solution $(\bar x, \bar z)$\;
$\phi \gets$ \textbf{KT-Round}$(\mathcal{C}, S_{k}, P', \bar x, \bar z)$\;
\caption{Approximating  \textbf{$k$-center-SPC-CC}}\label{alg-2}
\end{algorithm}

\begin{theorem}\label{s3-thm1}
Let $\tau^*_{C}$ the optimal value of the given \textbf{$k$-center-SPC-CC} instance, and $\mathcal{D}$ the distribution over assignments given by \textbf{KT-Round}. Then Algorithm \ref{alg-2} guarantees $|S_{k}| \leq k$,  $\sum_{\{j,j'\} \in P_q}\Pr_{\phi \sim \mathcal{D}}[\phi(j) \neq \phi(j')] \leq 2\psi_q |P_q| ~\forall P_q \in \mathcal{P}$, $\Pr_{\phi \sim \mathcal{D}}[d(\phi(j), j) \leq 3\tau^*_{C}] = 1$ for all $j \in \mathcal{C}$, and $\Pr_{\phi \sim \mathcal{D}}[\phi(i)=i] = 1$ for all $i \in S_{k}$.
\end{theorem}

\begin{proof}
At first, notice that since $S_{k}$ results from running $A_{k}$, it must be the case that $|S_{k}| \leq k$.

Furthermore, because the optimal value of the $k$-center instance is less than $\tau^{*}_{C}$, $A_{k}(\mathcal{C}, \mathcal{F}, \mathcal{L}, \tau^{*}_{C})$ will not return ``infeasible'', and $(S_{k}, \phi_{k})$ has an objective value $\tau_{k} \leq 2\tau^{*}_{C}$.

Given this, the reasoning in the proof of Theorem \ref{s2-thm1} yields $\sum_{\{j,j'\} \in P_q}\Pr_{\phi \sim \mathcal{D}}[\phi(j) \neq \phi(j')] \leq 2\psi_q |P_q|, ~\forall P_q \in \mathcal{P}$, and $\Pr_{\phi \sim \mathcal{D}}[d(\phi(j), j) \leq 3\tau^*_{C}] = 1$ for all $j \in \mathcal{C}$, if the $\hat x$  defined in that proof also satisfies (\ref{s3-cnt-centr}). Hence, the latter statement regarding $\hat x$ is all we need to verify.

A crucial property of $A_{k}$  when executed with a target radius $\tau^*_C$, is that for all $i, \ell \in S_{k}$ with $i \neq \ell$ we have $d(i, \ell) > 2 \tau^*_C$ \cite{Hochbaum1986}. Due to this, for all $i \in S_k$ and $i' \in S^*$ with $d(i,i') \leq \tau^*_C$, we have $i' \in N(i)$. Suppose otherwise, and let $i' \in N(\ell)$ for some other $\ell \in S_k$. By definition this gives $d(i', \ell) \leq d(i',i) \leq \tau^*_C$. Thus, $d(i,\ell) \leq d(i,i') + d(i', \ell) \leq 2\tau^*_C$. Finally, note that due to $\tau^*_C$ being the optimal value of $k$-center-SPC-CC, we have $\sum_{i': d(i,i') \leq \tau^*_C}x^*_{i',i} = 1$ for $i \in S_k$. Since $\{i' \in S^*: d(i,i') \leq \tau^*_C\} \subseteq N(i)$,  we finally get $\hat x_{i,i} = 1$.

To conclude, we show that $\Pr_{\phi \sim \mathcal{D}}[\phi(i) = i] = 1$ for every $i \in S_k$. However, this is obvious, because of the second property of Theorem \ref{round}, and $\bar x$ satisfying (\ref{s2-LP-1}), (\ref{s3-cnt-centr}).
\end{proof}

\subsection{A Reassignment Step for the Unrestricted and $k$-Constrained Location Setting}

We now demonstrate a reassignment procedure that can be used to correct the output of Algorithm \ref{alg-1}, in a way that satisfies the CC. Again, let $P_{\mathcal{L}}$ be any of the vanilla objective functions, and consider Algorithm \ref{alg-3}.

\begin{algorithm}[t]
Run Algorithm \ref{alg-1} to solve $P_\mathcal{L}$-SPC, and get $S \subseteq \mathcal{C}$ and an assignment $\phi: \mathcal{C} \mapsto S$ in return\;
\While {there exists $i \in S$ with $\phi(i) \neq i$} {
$S \gets S \setminus \{i\}$\;
$i' \gets \argmin_{j \in \mathcal{C}: \phi(j) = i}d(i,j)$\;
$S \gets S \cup \{i'\}$\;
\For {all $j \in \mathcal{C}$ with $\phi(j) = i$} {
$\phi(j) \gets i'$\;
}
}
\caption{Approximating  \textbf{$P_\mathcal{L}$-SPC-CC}}\label{alg-3}
\end{algorithm}

\begin{theorem}\label{s3-thm2}
Let $\lambda$ the approximation ratio of Algorithm \ref{alg-1} for \textbf{$P_\mathcal{L}$-SPC} with respect to the objective function. Then, Algorithm \ref{alg-3} gives an approximation ratio $2\lambda$ for the objective of \textbf{$P_\mathcal{L}$-SPC-CC}, while satisfying the CC and preserving the guarantees of Algorithm \ref{alg-1} on SPCs, when $\mathcal{L} = 2^\mathcal{C}$ or $\mathcal{L} = \{S' \subseteq \mathcal{C}:~ |S'| \leq k\}$ for some integer $k$.
\end{theorem}

\begin{proof}
To show that Algorithm \ref{alg-3} preserves the SPC guarantees, it suffices to prove that the following invariant is maintained throughout the reassignment. For any $j,j'$ that are assigned by Algorithm \ref{alg-1} to the same location, at the end of the while loop we should still have $\phi(j) = \phi(j')$. Suppose now that we initially had $\phi(j)=\phi(j') = i$. If $\phi(i) = i$, then no change will occur, and in the end we still have $\phi(j)=\phi(j') = i$. If on the other hand $\phi(i) \neq i$, we choose an $i' \in \{j'' \in \mathcal{C}: \phi(j'') = i\}$ and then set $\phi(j)=\phi(j') = i'$. Note that this assignment will not change, since the modification also ensures $\phi(i') = i'$. The latter also guarantees that CC holds at termination. 

Moreover, notice that when $\mathcal{L} = \{S' \subseteq \mathcal{C}:~ |S'| \leq k\}$, the cardinality constraint is not violated. The reason for this is that every time we possibly add a new location to $S$, we have already removed another one from there. 

Finally, we need to reason about the approximation ratio of Algorithm \ref{alg-3}. At first, let $i$ the location to which $j \in \mathcal{C}$ was assigned to by Algorithm \ref{alg-1}. As we have already described, $j$ can only be reassigned once, and let $i'$ be its new assignment. Then $d(i',j) \leq d(i,i') + d(i,j) \leq 2d(i,j)$, where the second inequality follows because $i' = \argmin_{j' \in \mathcal{C}: \phi(j') = i}d(i,j')$. Hence, we see that the covering distance for every point at most doubles, and since the optimal value of \textbf{$P_\mathcal{L}$-SPC} is not larger than that of \textbf{$P_\mathcal{L}$-SPC-CC}, we get an approximation ratio of $2\lambda$ for \textbf{$P_\mathcal{L}$-SPC-CC} (Notice that this holds regardless of the specific objective).
\end{proof}

\section{Improved Results for Problems with Must-Link Constraints}\label{sec:4}

Since must-link constraints (ML) are a special case of SPCs, Algorithm \ref{alg-1} provides approximation results for the former as well (also note that due to $\psi_p = 0 ~\forall p$, we have no pairwise constraint violation when using Algorithm \ref{alg-1} purely for ML). However, in this section we demonstrate how we can get improved approximation guarantees for some of the problems we consider. Specifically, we provide a $2/3/3/3$-approximation for $k$-center-ML/knapsack-center-ML/$k$-supplier-ML/knapsack-supplier-ML, which constitutes a clear improvement over the $3/4/5/5$-approximation, given when Algorithm \ref{alg-1} is executed using the best approximation algorithm for the corresponding vanilla variant. 

First of all, recall that in the ML case we are only looking for a set of locations $S$ and an assignment $\phi: \mathcal{C} \mapsto S$, and not for a distribution over assignments. Also, notice that the must-link relation is transitive. If for $j,j'$ we want $\phi(j) = \phi(j')$, and for $j',j''$ we also require $\phi(j') = \phi(j'')$, then $\phi(j) = \phi(j'')$ is necessary as well. Given that, we view the input as a partition $C_1, C_2, \hdots, C_t$ of the points of $\mathcal{C}$, where all points in $C_q$, with $q \in \{1,\hdots,t\}$, must be assigned to the same location of $S$. We call each part $C_i$ of this partition a clique. Finally, for the problems we study, we can once more assume w.l.o.g. that the optimal radius $\tau^*$ is known.

\begin{definition}\label{Neighbor}
Two cliques $C_q, ~ C_p$ are called \textbf{neighboring} if $\forall j \in C_q, ~ \forall j' \in C_p$ we have $d(j,j') \leq 2\tau^*$.
\end{definition}

Algorithm \ref{alg-4} captures $k$-center-ML, knapsack-center-ML, $k$-supplier-ML and knapsack-supplier-ML at once, yielding improved approximations for each of them.

\begin{algorithm}[t]
$C \gets \emptyset$, $S \gets \emptyset$\;
Initially all $C_1, C_2, \hdots, C_t$ are considered uncovered\;
\While{ there exists an uncovered $C_q$}{
Pick an uncovered $C_q$\;
Pick an arbitrary point $j_q \in C_q$\;
$C\gets C \cup \{j_q\}$\;
$C_q$ and all neighboring cliques $C_p$ of it, are now considered covered\;
}
\For {all $j_q \in C$}{
\If {$k$-center/$k$-supplier} {
$i_q \gets \argmin_{i \in \mathcal{F}}d(i,j_q)$\;$S \gets S \cup \{i_q\}$\;
}
\If {knapsack-center/knapsack-supplier} {
$i_q \gets \argmin_{i \in \mathcal{F}: d(i,j_q) \leq \tau^*}w_i$\; $S \gets S \cup \{i_q\}$\;
}
}
\For {all $j \in \mathcal{C}$}{
Let $j_q \in C$ the point whose clique $C_q$ covered $j$'s clique in the first while loop\;
$\phi(j) \gets i_q$\;
}
\caption{Approximating ML Constraints}\label{alg-4}
\end{algorithm}

\begin{theorem}
Algorithm \ref{alg-4} is a $2/3/3/3$-approximation algorithm for \textbf{$k$-center-ML}/\textbf{knapsack-center-ML}/\textbf{$k$-supplier-ML}/\textbf{knapsack-supplier-ML}.
\end{theorem}

\begin{proof}
Initially, observe that the must-link constraints are satisfied. When the algorithm chooses a location $i_q$ based on some $j_q \in C$, all the points in $C_q$ are assigned to $i_q$. Also, for the neighboring cliques of $C_q$ that got covered by it, their whole set of points ends up assigned to $i_q$ as well.

We now argue about the achieved approximation ratio. For one thing, it is clear that for every $j,j' \in C_q$ we must have $d(j,j') \leq 2\tau^*$, and therefore $C_q$ is a neighboring clique of itself. We thus have the following two cases:
\begin{itemize}
    \item \textbf{$k$-center-ML:} Here for each $j_q \in C$, we end up placing $j_q$ in $S$ as well. Also, all points $j$ assigned to $j_q$ belong to neighboring cliques of $C_q$, and therefore $d(j,j_q) \leq 2\tau^*$.
    \item \textbf{In all other problems}, for each $j_q \in C$ we choose a location $i_q$ such that $d(i_q, j_q) \leq \tau^*$. Also, all points $j$ assigned to $i_q$ belong to neighboring cliques of $C_q$, and therefore $d(i_q,j) \leq d(i_q,j_q) + d(j_q,j) \leq 3\tau^*$.
\end{itemize}

Finally, we need to show that either the cardinality or the knapsack constraint on the set of chosen locations is satisfied. Toward that end, notice that if $C_q$ and $C_p$ belong in the same cluster in the optimal solution, then they are neighboring. Say $i^{\star}$ is the location they are both assigned to. Then for all $j\in C_q$ and $j' \in C_p$ we get $ d(j,i^{\star}) \leq \tau^*$ and $d(j',i^{\star}) \leq \tau^*$. Hence, by the triangle inequality for all $j \in C_q$ and $j' \in C_p$ we have $d(j,j') \leq 2\tau^*$.

Given the previous observation, it must be the case that for every $j_q \in C$, the optimal solution assigns it to a location $i^*_{j_q}$, such that for every other $j_{q'} \in C$ we have $i^*_{j_q} \neq i^*_{j_{q'}}$. Therefore, in the presence of a cardinality constraint:
\begin{align}
    |S| \leq |C| = \sum_{i^*_{j_q}: ~j_q \in C}1 \leq k \notag
\end{align}
and in the presence of a knapsack constraint:
\begin{align}
    \sum_{i \in S}w_i \leq \sum_{j_q \in C}\min_{i \in \mathcal{F}: d(i,j_q) \leq \tau^*}w_i \leq \sum_{j_q \in C}w_{i^*_{j_q}} \leq W \notag
\end{align}
where the last inequality in both cases follows from the optimal solution satisfying the corresponding constraint.
\end{proof}

\begin{observation}
Algorithm \ref{alg-4} is a $2/3$-approximation for $k$-center-ML-CC/knapsack-center-ML-CC. This directly follows from steps 10, 13 and 17 of it.
\end{observation}

\begin{observation}
Due to known hardness results for the vanilla version of the corresponding problems \cite{Hochbaum1986}, Algorithm \ref{alg-4} gives the best possible approximation ratios, assuming that $P \neq NP$.
\end{observation}

\section{Experimental Evaluation}

\begin{table}[t]
\centering
\begin{tabular}{|c|l|l|l|l|l|}
\hline
\multicolumn{1}{|l|}{}                &  k  & 4     & 6     & 8     & 10     \\ \hline
\multirow{2}{*}{Adult}                            & Alg-1 & 2.27  & 4.73  & 12.53 & 21.81  \\ \cline{2-6} 
                                                  & ALG-IF      & 84.87 & 91.76 & 100.00 & 100.00  \\ \hline
\multirow{2}{*}{Bank}                             & Alg-1 & 0.16  & 0.44  & 0.34  & 0.54   \\ \cline{2-6} 
                                                  & ALG-IF & 55.84 & 71.48 & 92.85 & 99.93  \\ \hline
\multicolumn{1}{|l|}{\multirow{2}{*}{Credit}} & Alg-1 & 1.34 & 2.58 & 9.03 & 14.76  \\ \cline{2-6} 
\multicolumn{1}{|l|}{}                            & ALG-IF & 80.25 & 100.00 & 100.00  & 100.00 \\ \hline
\end{tabular}
\captionsetup{justification=centering}
\caption{Percentage of constraints that are violated on average for metric $F_2$}
\label{tab:fairness-f2}
\end{table}

\begin{table}[t]
\centering
\begin{tabular}{|c|l|l|l|l|l|} 
\hline
\multicolumn{1}{|l|}{}                & k   & 4    & 6    & 8    & 10    \\ 
\hline
\multirow{2}{*}{Adult}                            & Alg-1 & 1.88 & 2.41 & 3.09 & 3.48  \\ 
\cline{2-6}
                                                  & ALG-IF      & 1.88 & 2.38 & 3.21 & 3.44  \\ 
\hline
\multirow{2}{*}{Bank}                             & Alg-1 & 2.34 & 3.28 & 4.09 & 4.62  \\ 
\cline{2-6}
                                                  & ALG-IF & 2.36 & 3.34 & 4.67 & 4.93  \\ 
\hline
\multicolumn{1}{|l|}{\multirow{2}{*}{Credit}} & Alg-1 & 1.82 & 2.12 & 2.46 & 2.71  \\ 
\cline{2-6}
\multicolumn{1}{|l|}{}                            & ALG-IF & 1.80 & 2.20 & 2.43 & 2.66  \\
\hline
\end{tabular}
\captionsetup{justification=centering}
\caption{Cost of fairness for metric $F_2$}
\label{tab:cost-f2}
\end{table}

We implement our algorithms in Python 3.8 and run our experiments on AMD Opteron 6272 @ 2.1 GHz with 64 cores and 512 GB 1333 MHz DDR3 memory. We focus on fair clustering applications with PBS constraints and evaluate against the most similar prior work. Comparing to \cite{anderson2020} using $k$-means-PBS, shows that our algorithm violates fewer constraints while achieving a comparable cost of fairness. Similarly, our comparison with \cite{brubach2020} using $k$-center-PBS-CC (that prior algorithm also satisfies the CC constraint) reveals that we are better able to balance fairness constraints and the objective value. Our code is publicly available at \url{https://github.com/chakrabarti/pairwise_constrained_clustering}.

\textbf{Datasets: } We use 3 datasets from the UCI ML Repository \cite{Dua:2019}: \textbf{(1)} Bank-4,521 points \cite{moro}, \textbf{(2)} Adult-32,561 points \cite{kohavi}, and \textbf{(3)} Creditcard-30,000 points \cite{yeh}. 

\begin{table}[t]
\centering
\begin{tabular}{|c|l|l|l|l|l|} 
\hline
\multicolumn{1}{|l|}{}                & k   & 4    & 6    & 8    & 10    \\ 
\hline
\multirow{2}{*}{Adult}                            & Alg-1 & 0.14 & 0.25 & 0.58 & 0.77  \\ 
\cline{2-6}
                                                  & ALG-IF      & 7.98 & 7.07 & 8.90 & 9.42  \\ 
\hline
\multirow{2}{*}{Bank}                             & Alg-1 & 0.02 & 0.16 & 0.20 & 0.50  \\ 
\cline{2-6}
                                                  & ALG-IF & 4.25 & 5.09 & 5.58 & 6.37  \\ 
\hline
\multicolumn{1}{|l|}{\multirow{2}{*}{Credit}} & Alg-1 & 0.00 & 0.07 & 0.21 & 0.25  \\ 
\cline{2-6}
\multicolumn{1}{|l|}{}                            & ALG-IF & 0.97 & 3.80 & 4.17 & 3.90  \\
\hline
\end{tabular}
\captionsetup{justification=centering}
\caption{Percentage of constraints that are violated on average for metric $F_3$}
\label{tab:fairness-f3}
\end{table}

\begin{table}[t]
\centering

\begin{tabular}{|c|l|l|l|l|l|} 
\hline
\multicolumn{1}{|l|}{} & k   & 4    & 6    & 8    & 10    \\ 
\hline
\multirow{2}{*}{Adult}             & Alg-1 & 1.13 & 1.20 & 1.23 & 1.24  \\ 
\cline{2-6}
                                   & ALG-IF      & 1.13 & 1.20 & 1.23 & 1.23  \\ 
\hline
\multirow{2}{*}{Bank}              & Alg-1 & 1.22 & 1.36 & 1.41 & 1.44  \\ 
\cline{2-6}
                                   & ALG-IF      & 1.24 & 1.41 & 1.46 & 1.54  \\ 
\hline
\multirow{2}{*}{Credit}        & Alg-1 & 1.12 & 1.11 & 1.10 & 1.10  \\ 
\cline{2-6}
                                   & ALG-IF & 1.10 & 1.09 & 1.10 & 1.11  \\
\hline
\end{tabular}
\captionsetup{justification=centering}
\caption{Cost of fairness for metric $F_3$}
\label{tab:cost-f3}
\end{table}

\begin{table}[t]

\centering

\begin{tabular}{|c|l|l|l|l|l|l|l|} 
\hline
\multicolumn{1}{|l|}{} & k & 10 & 20 & 30    & 40    & 50    & 60    \\ 
\hline
\multirow{2}{*}{Adult}             & Alg-2 & .01 & .01 & .01 & .02 & .02 & .04 \\ 
\cline{2-8}
                                   & Alg-F(A) & .00  & .00  & .00  & .00 & .00 & .00  \\ 
\cline{2-8}
                                   & Alg-F(B) & .19 & .18 & .23 & .30 & .27 & .30 \\ 
\hline
\multirow{2}{*}{Bank}              & Alg-2 & .00 & .01 & .01 & .06 & .09 & .03  \\ 
\cline{2-8}
                                   & Alg-F(A) & .00  & .00 & .00 & .00 & .00 & .00  \\
\cline{2-8}
                                   & Alg-F(B) & .18 & .20 & .16 & .15 & .20 & .23 \\ 
\hline
\multirow{2}{*}{Credit}        & Alg-2 & .00 & .01 & .01 & .01 & .01 & .02 \\ 
\cline{2-8}
                                   & Alg-F(A) & .00  & .00 & .00 & .00 & .00 & .00  \\
\cline{2-8}
                                   & Alg-F(B) & .03 & .05 & .05 & .05 & .08 & .08 \\
\hline
\end{tabular}
\captionsetup{justification=centering}
\caption{Percentage of constraints that are violated on average for metric $F_1$}
\label{tab:fairness-f1}
\end{table}

\begin{table}[t]

\centering

\begin{tabular}{|c|l|l|l|l|l|l|l|} 
\hline
\multicolumn{1}{|l|}{} & k & 10 & 20 & 30    & 40    & 50    & 60    \\ 
\hline
\multirow{2}{*}{Adult}             & Alg-2 & .39  & .28 & .23 & .20 & .17 & .16  \\ 
\cline{2-8}
                                   & Alg-F(A) & .54 & .48 & .46 & .46 & .43 & .42 \\ 
\cline{2-8}
                                   & Alg-F(B) & .31 & .24 & .17 & .17 & .12 & .14 \\ 
\hline
\multirow{2}{*}{Bank}              & Alg-2 & .17 & .11 & .08 & .07 & .06 & .05 \\ 
\cline{2-8}
                                   & Alg-F(A) & .21  & .20 & .17 & .16 & .14 & .13 \\ 
\cline{2-8}
                                   & Alg-F(B) & .12 & .07 & .06 & .05 & .04 & .03 \\ 
\hline
\multirow{2}{*}{Credit}        & Alg-2 & .38 & .29 & .25 & .24 & .21 & .19 \\ 
\cline{2-8}
                                   & Alg-F(A) & .45  & .45 & .43 & .42 & .41 & .41  \\ 
\cline{2-8}
                                   & Alg-F(B) & .28 & .25 & .21 & .20 & .18 & .17 \\
\hline
\end{tabular}
\captionsetup{justification=centering}
\caption{Objective achieved for metric $F_1$}
\label{tab:cost-f1}
\end{table}

\textbf{Algorithms: } 
In all of our experiments, $\mathcal{C} = \mathcal{F}$ at first. When solving $k$-means-PBS, we use Lloyd's algorithm in the first step of Algorithm \ref{alg-1} and get a set of points $L$. The set of chosen locations $S$ is constructed by getting the nearest point in $\mathcal{C}$ for every point of $L$. This is exactly the approach used in \cite{anderson2020}, where their overall algorithm is called ALG-IF. To compare Algorithm \ref{alg-1} to ALG-IF, we use independent sampling for ALG-IF, in order to fix the assignment of each $j \in \mathcal{C}$ to some $i \in S$, based on the distribution $\phi_j$ produced by ALG-IF. For $k$-center-PBS-CC, we use Algorithm \ref{alg-2} with a binary search to compute $\tau^*_C$.

\textbf{Fairness Constraints: } We consider three similarity metrics ($F_1, F_2, F_3$) for generating PBS constraints. We use $F_1$ for $k$-center-PBS-CC and $F_2$, $F_3$ for $k$-means-PBS. $F_1$ is the metric used for fairness in the simulations of \cite{brubach2020} and $F_2, F_3$ are the metrics used in the experimental evaluation of the algorithms in \cite{anderson2020}. 

$F_1$ involves setting the separation probability between a pair of points $j$ and $j'$ to $d(j,j')/R_{Scr}$ if $d(j,j') \leq R_{Scr}$, where $R_{Scr}$ is the radius given by running the Scr algorithm \cite{Mihelic2005SolvingTK} on the provided input. 

$F_2$ is defined so that the separation probability between a pair $j$,$j'$ is given by $d(j,j')$, scaled linearly to ensure all such probabilities are in $[0, 1]$. Adopting the approach taken by \cite{anderson2020} when using this metric, we only consider pairwise constraints between each $j$ and its closest $m$ neighbors. For our experiments, we set $m = 100$. 

Again in order to compare our Algorithm \ref{alg-1} with \cite{anderson2020}, we need the metric $F_3$. For any $j \in \mathcal{C}$, let $r_j$ the minimum distance such that $|j' \in \mathcal{C}: d(j,j') \leq r_j| \geq |\mathcal{C}|/ k$. Then the separation probability between $j$ and any $j'$ such that $d(j,j') \leq r_j$, is set to $d(j,j')/r_j$.

\textbf{Implementation Details: }As performed in \cite{anderson2020,brubach2020}, we uniformly sample $N$ points from each dataset and run all algorithms on those sets, while only considering a subset of the numerical attributes and normalizing the features to have zero mean and unit variance. In our comparisons with \cite{anderson2020} we use $N= 1000$, while in our comparisons with \cite{brubach2020} $N$ is set to $250$. For the number of clusters $k$, we study the values $\{4, 6, 8, 10\}$ when comparing to \cite{anderson2020}, and $\{10, 20, 30, 40, 50, 60\}$ when comparing to \cite{brubach2020} (theoretically Algorithm \ref{alg-2} is better for larger $k$). Finally, to estimate the empirical separation probabilities and the underlying objective function cost, we run 5000 trials for each randomized assignment procedure, and then compute averages for the necessary performance measures we are interested in.

\textbf{Comparison with \cite{anderson2020}: }In Tables \ref{tab:fairness-f2} and \ref{tab:fairness-f3}, we show what percentage of fairness constraints are violated by ALG-IF and our algorithm, for the fairness constraints induced by $F_2$ and $F_3$, allowing for an $\epsilon = 0.05$ threshold on the violation of a separation probability bound; we only consider a pair's fairness constraint to be violated if the empirical probability of them being separated exceeds that set by the fairness metric by more than $\epsilon$. It is clear that our algorithm outperforms ALG-IF consistently across different values of $k$, different datasets, and both types of fairness constraints considered by \cite{anderson2020}.

In order to compare the objective value achieved by both algorithms, we first compute the average connection costs over the 5000 runs. Since the cost of the clustering returned by Lloyd's algorithm contributes to both Algorithm \ref{alg-1} and ALG-IF, we utilize that as an approximation of the cost of fairness. In other words, we divide the objective value of the final solutions by the cost of the clustering produced by Lloyd, and call this quantity cost of fairness. The corresponding comparisons are presented in Tables \ref{tab:cost-f2}, \ref{tab:cost-f3}. The cost of fairness for both algorithms is very similar, demonstrating a clear advantage of Algorithm \ref{alg-1}, since it dominates ALG-IF in the percentage of fairness constraints violated.

\textbf{Comparison with \cite{brubach2020}: } In Table \ref{tab:fairness-f1} we show what percentage of fairness constraints are violated by the algorithm of \cite{brubach2020} (named Alg-F) and Algorithm \ref{alg-2}, using an $\epsilon = 0$; if the empirical probability of separation of a pair exceeds the bound set by the fairness metric by any amount, it is considered a violation. We run ALG-F with two different choices of the scale parameter used in that prior work: $\frac{1}{R_{Scr}}$ (Alg-F(A)) and $\frac{16}{R_{Scr}}$ (Alg-F(B)), where $R_{Scr}$ is the value achieved using the Scr algorithm. The reason for doing so is that \cite{brubach2020} consider multiple values for the separation probabilities, and we wanted to have a more clear comparison of our results against all of those. Alg-F(A) leads to 0 violations, while our algorithm produces a small number of violations in a few cases, and Alg-F(B) leads to a significant number of violations. In Table \ref{tab:cost-f1}, we show the cost of the clusterings produced by ALG-F and Algorithm \ref{alg-2}, measured in the normalized metric space by taking the average of the maximum radius of any cluster over the 5000 runs. Alg-F(b) leads to the lowest objective, followed relatively closely by our algorithm, and then finally Alg-F(A) has significantly higher objective values.


\section*{Acknowledgements}
The authors would like to sincerely thank Samir Khuller for useful discussions that led to some of the technical results of this work. In addition, we thank Bill Gasarch for his devotion to building a strong REU program, which facilitated coauthor Chakrabarti's collaboration. Finally, we thank the anonymous referees for multiple useful suggestions.

Brian Brubach was  supported in part by NSF award CCF-1749864. 
John Dickerson was supported in part by NSF CAREER Award IIS-1846237, NSF Award CCF-1852352, NSF D-ISN Award \#2039862, NIST MSE Award \#20126334, NIH R01 Award NLM-013039-01, DARPA GARD Award \#HR00112020007, DoD WHS Award \#HQ003420F0035, DARPA Disruptioneering Award (SI3-CMD) \#S4761 and Google Faculty Research Award. Aravind Srinivasan was supported in part by NSF awards CCF-1422569, CCF-1749864, and CCF-1918749, as well as research awards from Adobe, Amazon, and Google. Leonidas Tsepenekas was  supported in part by NSF awards CCF-1749864 and CCF-1918749, and by research awards from Amazon and Google. 

\section*{Ethics Statement}
Our primary contribution is general and theoretical in nature, so we do not foresee any immediate and direct negative ethical impacts of our work.  That said, one use case of our framework---that we highlight prominently both in the general discussion of theoretical results as well as through experimental results performed on standard and commonly-used datasets---is as a tool to operationalize notions of \emph{fairness} in a broad range of clustering settings.  Formalization of fairness as a mathematical concept, while often grounded in legal doctrine~\citep[see, e.g.,][]{Feldman15:Certifying,Barocas19:Fairness}, is still a morally-laden and complicated process, and one to which there is no one-size-fits-all ``correct'' approach.

Our method supports a number of commonly-used fairness definitions; thus, were tools built based on our framework that operationalized those definitions of fairness, then the ethical implications---both positive and negative---of that decision would also be present.  Our framework provides strong theoretical guarantees that would allow decision-makers to better understand the performance of systems built based on our approach. Yet, we also note that any such guarantees should, in many domains, be part of a larger conversation with stakeholders---one including understanding the level of comprehension~\citep[e.g.,][]{Saha20:Measuring,Saxena20:How} and specific wants of stakeholders~\citep[e.g.,][]{Holstein19:Improving,Madaio20:Co-Designing}.
\bibliography{refs}

\appendix
\onecolumn
\section{Omitted Details}\label{appendix}

\paragraph{Proof of Theorem \ref{np-hard}}
\begin{proof}
Our reduction is based on the $k$-cut problem, in which we are given an undirected graph $G(V,E)$, $k$ distinct nodes $u_1, u_2, \hdots, u_k \in V$ and a $\gamma \in \mathbb{N}_{\geq 0}$. Then we want to know whether or not there exists a cut $S \subseteq E$ with $|S| \leq \gamma$, such that in $G(V, E \setminus S)$ the nodes $u_1, u_2, \hdots, u_k$ are separated (each of them is placed in a different connected component).\\

\noindent \textbf{Supplier/Median/Means:} We now describe the construction of an instance for the clustering problem. For each of the $k$ distinct nodes $u_f$ we create a location $i_{u_f}$, and for every two different locations $i_{u_f}, i_{u_{f'}}$ we set $d(i_{u_f}, i_{u_{f'}}) = 2$. Then, again for each of the $k$ distinct nodes $u_f$ we create a point $j_{u_f}$ that is co-located with $i_{u_f}$. In addition, for each of the $|V| - k$ remaining nodes $u$ of the graph we create a point $j_u$, such that $d(i_{u_f}, j_u) = 1$ for every previously created location $i_{u_f}$, and all those new points are co-located. This construction constitutes a valid metric space. As for the stochastic pairwise constraints, we only have one set of pairs $P$. For each $\{u,v\} \in E$ let $j_u, j_v$ the corresponding points in the created metric, and then add $\{j_u,j_v\}$ to $P$. In the end, we ask for a solution that in expectation separates at most $\gamma$ of the pairs of $P$ (i.e. $\psi_P = \gamma / |E|$).

We now claim that the clustering problem has a solution of objective value at most $1/|V|-k/\sqrt{|V|-k}$ in the constructed instance for supplier/median/means iff there exists a $k$-cut of size at most $\gamma$ in $G(V,E)$. Before we proceed, notice that this claim for the median and means objectives implies that there is no actual randomness in the assignment cost. This is true because $|V|-k, \sqrt{|V|-k}$ are the least possible values for the corresponding objectives, resulting deterministically from the assignment of the $|V|-k$ points that do not correspond to nodes that need to be separated. 

Therefore, suppose at first that there exists a $k$-cut $S$ of size at most $\gamma$. We construct a clustering solution as follows. Open all locations $i_{u_f}$ of the instance. Then for each $u \in V$ with corresponding point $j_u$, assign $j_u$ to $i_{u_f}$, such that $u$ is in the same connected component in $G(V, E \setminus S)$ as ${u_f}$. This clearly achieves deterministically the desired objective value cost for each case. Moreover, each separated pair of $P$ corresponds to an edge of $S$, and hence in the end the number of separations in this clustering solution will be upper bounded by $|S| \leq \gamma$.

On the other hand, assume that there exists a solution to the clustering instance of cost at most $1/|V|-k/\sqrt{|V|-k}$. This immediately implies that all available locations are opened. Furthermore, since on expectation at most $\gamma$ pairs are separated, there must be an assignment $\phi$ that separates less than $\gamma$. Given this assignment, we construct a $k$-cut as follows. Remove from $E$ all edges $\{u,v\}$ such that the corresponding points $j_u,j_v$ have $\phi(j_u) \neq \phi(j_v)$. This clearly removes at most $\gamma$ edges from $E$. To see that two of the special nodes $u_f, u_{f'}$ will be separated after the above edge removal, note that their respective points in the metric space are assigned to different clusters. Also, the corresponding node of each point in a cluster cannot have an edge to the corresponding node of a point in a different cluster, and hence $u_f, u_{f'}$ are definitely in distinct connected components.\\

\noindent \textbf{Center:} To handle this objective function, we use an instance similar to that constructed for the supplier objective, with just one difference. This difference will make sure that we can only open locations that correspond to points coinciding with the locations used earlier. So in the previous instance, for each $j_{u_f}$ corresponding to a special node $u_f$ we create another point $h_{u_f}$, such that $d(j_{u_f}, h_{u_f}) = 1$, $d(j_u, h_{u_f}) = 2$ for all $j_u \neq j_{u_f}$ with $j_u$ resulting from some $u \in V$, and $d(h_{u_f}, h_{u_{f'}})=2$ for every $h_{u_{f'}}$ resulting from $u_{f'} \neq u_f$. Then, the arguments for the supplier case go through in the exact same manner, since the only locations that can be considered for opening when aiming for a solution of radius $1$, are the points $j_{u_f}$.

\end{proof}

\end{document}